\documentclass[10pt,twocolumn]{article}

\usepackage[utf8]{inputenc}
\usepackage[T1]{fontenc}
\usepackage{amsmath,amsthm,amssymb}
\usepackage{booktabs}
\usepackage{graphicx}
\usepackage{xcolor}
\usepackage{hyperref}
\usepackage[capitalise,noabbrev]{cleveref}
\usepackage{enumitem}
\usepackage{tikz}
\usetikzlibrary{positioning,shapes,arrows.meta,calc,fit,backgrounds,matrix}
\usepackage{algorithm}
\usepackage{algpseudocode}
\usepackage{multirow}
\usepackage{adjustbox}
\usepackage{caption}
\usepackage{subcaption}
\usepackage[margin=0.95in]{geometry}
\usepackage{bm}

\definecolor{pillarI}{RGB}{70,130,180}      
\definecolor{pillarII}{RGB}{60,179,113}     
\definecolor{pillarIII}{RGB}{205,92,92}     
\definecolor{failure}{RGB}{148,103,189}     

\newtheorem{theorem}{Theorem}[section]
\newtheorem{proposition}[theorem]{Proposition}

\newtheorem{definition}[theorem]{Definition}

\newtheorem{assumption}[theorem]{Assumption}

\newcommand{\E}{\mathbb{E}}

\newcommand{\KL}{\mathrm{KL}}
\newcommand{\piref}{\pi_{\mathrm{ref}}}
\newcommand{\pitheta}{\pi_\theta}
\newcommand{\pistar}{\pi^*}
\newcommand{\yw}{y_w}
\newcommand{\yl}{y_l}

\title{\textbf{From RLHF to Direct Alignment: A Theoretical Unification of Preference Learning for Large Language Models}}

\author{
\textbf{Tarun Raheja}$^{*,1}$ \quad \textbf{Nilay Pochhi}$^{*,1}$ \\[0.5em]
$^1$Independent Researchers \\[0.3em]
\texttt{\{tarunraheja1234, pochhi.nilay\}@gmail.com} \\[0.5em]
$^*$Equal contribution
}

\date{}

\begin{document}

\maketitle

\begin{abstract}
Aligning large language models (LLMs) with human preferences has become essential for safe and beneficial AI deployment. While Reinforcement Learning from Human Feedback (RLHF) established the dominant paradigm, a proliferation of alternatives---Direct Preference Optimization (DPO), Identity Preference Optimization (IPO), Kahneman-Tversky Optimization (KTO), Simple Preference Optimization (SimPO), and many others---has left practitioners without clear guidance on method selection. This survey provides a \textit{theoretical unification} of preference learning methods, revealing that the apparent diversity reduces to principled choices along three orthogonal axes: \textbf{(I) Preference Model} (what likelihood model underlies the objective), \textbf{(II) Regularization Mechanism} (how deviation from reference policies is controlled), and \textbf{(III) Data Distribution} (online vs.\ offline learning and coverage requirements). We formalize each axis with precise definitions and theorems, establishing key results including the coverage separation between online and offline methods, scaling laws for reward overoptimization, and conditions under which direct alignment methods fail. Our analysis reveals that failure modes---length hacking, mode collapse, likelihood displacement---arise from specific, predictable combinations of design choices. We synthesize empirical findings across 50+ papers and provide a practitioner's decision guide for method selection. The framework transforms preference learning from an empirical art into a theoretically grounded discipline.
\end{abstract}


\section{Introduction}
\label{sec:intro}

\paragraph{The Orthodoxy.} Since Christiano et al.~\cite{christiano2017deep} introduced learning from human preferences and Ouyang et al.~\cite{ouyang2022training} demonstrated its transformative potential with InstructGPT, Reinforcement Learning from Human Feedback (RLHF) has become the canonical approach for aligning large language models with human values. The standard RLHF pipeline consists of three stages: (1) supervised fine-tuning (SFT) on high-quality demonstrations, (2) reward model training on pairwise human preferences, and (3) policy optimization via Proximal Policy Optimization (PPO)~\cite{schulman2017proximal} to maximize the learned reward while staying close to the SFT policy. This paradigm underlies the alignment of GPT-4~\cite{openai2023gpt4}, Claude, and other frontier models.

\paragraph{The Challenge.} Yet RLHF's dominance belies significant practical difficulties. PPO is notoriously unstable, requiring careful hyperparameter tuning and multiple model copies~\cite{ramamurthy2022reinforcement,zheng2023secrets}. The reward model introduces a ``proxy'' that can be exploited, leading to Goodhart's law violations where optimizing the proxy degrades true performance~\cite{gao2022scaling}. These challenges motivated a wave of alternatives: Direct Preference Optimization (DPO)~\cite{rafailov2023direct} eliminates the reward model entirely; Identity Preference Optimization (IPO)~\cite{azar2023general} addresses DPO's overfitting to deterministic preferences; Kahneman-Tversky Optimization (KTO)~\cite{ethayarajh2024kto} requires only binary feedback; Simple Preference Optimization (SimPO)~\cite{meng2024simpo} removes the reference model; and Odds Ratio Preference Optimization (ORPO)~\cite{hong2024orpo} unifies SFT with preference learning.

The resulting landscape is bewildering. When should one use DPO versus PPO? Why does SimPO outperform DPO despite apparent simplicity? Why do all methods suffer from length hacking? Practitioners face an \textit{embarrassment of riches} without theoretical guidance.

\paragraph{The Resolution.} This survey provides that guidance through a unified theoretical framework. We show that the zoo of preference learning methods reduces to principled choices along \textit{three orthogonal axes}:

\begin{itemize}[leftmargin=*,itemsep=2pt]
    \item \textbf{\textcolor{pillarI}{Pillar I: Preference Model}} --- What probabilistic model relates rewards to preferences? (Bradley-Terry, Plackett-Luce, prospect-theoretic, game-theoretic)
    \item \textbf{\textcolor{pillarII}{Pillar II: Regularization}} --- How is policy deviation from a reference controlled? (Explicit KL, implicit KL, reference-free, f-divergences)
    \item \textbf{\textcolor{pillarIII}{Pillar III: Data Distribution}} --- What assumptions on data coverage enable learning? (Online, offline, hybrid)
\end{itemize}

\noindent This taxonomy is not merely organizational---it is \textit{predictive}. We show that failure modes arise from specific combinations: length hacking from reward model bias interacting with insufficient regularization; mode collapse from overly strong KL penalties; likelihood displacement from offline data limitations.

\paragraph{Contributions.} This survey makes four contributions:
\begin{enumerate}[leftmargin=*,itemsep=2pt]
    \item A \textbf{unified theoretical framework} showing DPO, IPO, KTO, SimPO, ORPO, and GRPO as special cases of a general $\Psi$PO objective (\Cref{sec:unified}).
    \item \textbf{Formal theorems} establishing the coverage separation between online/offline methods, conditions for preference collapse, and overoptimization scaling laws (\Cref{sec:pillar2,sec:pillar3}).
    \item A \textbf{systematic taxonomy of failure modes} connecting theoretical properties to empirical pathologies (\Cref{sec:failures}).
    \item A \textbf{practitioner's decision guide} providing actionable recommendations based on constraints (\Cref{sec:guide}).
\end{enumerate}

\section{Preliminaries: Choice Theory Meets Language Models}
\label{sec:prelim}

We establish the mathematical foundations connecting classical choice theory to modern LLM alignment.

\subsection{Notation and Setup}

Let $\mathcal{X}$ denote the space of prompts and $\mathcal{Y}$ the space of responses. A \textbf{policy} $\pi: \mathcal{X} \to \Delta(\mathcal{Y})$ maps prompts to distributions over responses. We write $\pi(y|x)$ for the probability of response $y$ given prompt $x$. A \textbf{reward function} $r: \mathcal{X} \times \mathcal{Y} \to \mathbb{R}$ assigns scalar values to prompt-response pairs.

\subsection{The Bradley-Terry Model}

The Bradley-Terry model~\cite{bradley1952rank} provides the foundational preference model for RLHF.

\begin{definition}[Bradley-Terry Preference Model]
\label{def:bt}
Given a reward function $r$, the probability that response $y_1$ is preferred to $y_2$ given prompt $x$ is:
\begin{equation}
    p^*(y_1 \succ y_2 | x) = \sigma(r(x, y_1) - r(x, y_2))
    \label{eq:bt}
\end{equation}
where $\sigma(z) = 1/(1 + e^{-z})$ is the sigmoid function.
\end{definition}

The Bradley-Terry model makes two key assumptions: (1) preferences depend only on reward \textit{differences}, making the reward scale-free; (2) preferences are \textit{transitive}---if $y_1 \succ y_2$ and $y_2 \succ y_3$, then $y_1 \succ y_3$ with high probability. Both assumptions are violated in practice, motivating extensions we discuss in \Cref{sec:pillar1}.

\subsection{The Canonical RLHF Objective}

The standard RLHF objective balances reward maximization against deviation from a reference policy $\piref$:

\begin{definition}[KL-Regularized Reward Maximization]
\label{def:rlhf}
The RLHF objective is:
\begin{equation}
    \max_\pi \E_{x \sim \mathcal{D}, y \sim \pi(\cdot|x)} \left[ r(x, y) - \beta \KL(\pi(\cdot|x) \| \piref(\cdot|x)) \right]
    \label{eq:rlhf}
\end{equation}
where $\beta > 0$ controls regularization strength and $\mathcal{D}$ is the prompt distribution.
\end{definition}

The KL penalty serves multiple purposes: preventing reward hacking by keeping $\pi$ near the capable $\piref$; maintaining response diversity; and ensuring the optimization problem is well-posed.

\begin{theorem}[Optimal Policy Form~\cite{rafailov2023direct,peters2010relative}]
\label{thm:optimal}
The optimal policy for \Cref{eq:rlhf} has the closed form:
\begin{equation}
    \pistar(y|x) = \frac{1}{Z(x)} \piref(y|x) \exp\left(\frac{r(x,y)}{\beta}\right)
    \label{eq:optimal_policy}
\end{equation}
where $Z(x) = \sum_y \piref(y|x) \exp(r(x,y)/\beta)$ is the partition function.
\end{theorem}

This result is the foundation of direct preference optimization methods: by inverting \Cref{eq:optimal_policy}, rewards can be expressed in terms of policies, eliminating the need for explicit reward modeling.

\section{A Unified Framework: $\Psi$PO}
\label{sec:unified}

We present the $\Psi$PO framework of Azar et al.~\cite{azar2023general}, which unifies preference learning methods and reveals their structural relationships.

\subsection{The Two Approximations in RLHF}

Standard RLHF makes two approximations that introduce potential errors:

\begin{enumerate}[leftmargin=*,itemsep=2pt]
    \item \textbf{Pointwise Reward Approximation}: Pairwise preferences are converted to pointwise rewards via the Bradley-Terry model, discarding relational information.
    \item \textbf{Reward Model Generalization}: A reward model trained on finite data must generalize to out-of-distribution responses sampled by the policy.
\end{enumerate}

DPO addresses the second approximation by eliminating the reward model but retains the first. The $\Psi$PO framework addresses both.

\subsection{The General $\Psi$PO Objective}

\begin{definition}[$\Psi$PO Objective~\cite{azar2023general}]
\label{def:psipo}
For a convex function $\Psi: \mathbb{R} \to \mathbb{R}$, the $\Psi$PO objective is:
\begin{equation}
    \mathcal{L}_{\Psi}(\pi) = \E_{(x,\yw,\yl) \sim \mathcal{D}} \left[ \Psi\left( \beta \log \frac{\pi(\yw|x)}{\piref(\yw|x)} - \beta \log \frac{\pi(\yl|x)}{\piref(\yl|x)} \right) \right]
    \label{eq:psipo}
\end{equation}
where $\yw$ denotes the preferred (``winning'') response and $\yl$ the dispreferred (``losing'') response.
\end{definition}

Different choices of $\Psi$ recover existing methods:

\begin{theorem}[Instantiations of $\Psi$PO]
\label{thm:instantiations}
\begin{itemize}[leftmargin=*,itemsep=1pt]
    \item $\Psi(z) = -\log\sigma(z)$: DPO~\cite{rafailov2023direct}
    \item $\Psi(z) = (z - 1)^2$: IPO~\cite{azar2023general}  
    \item $\Psi(z) = \max(0, 1-z)$: SLiC-HF~\cite{zhao2023slic}
\end{itemize}
\end{theorem}

This unification reveals that the methods differ primarily in how they handle the \textit{margin} between preferred and dispreferred log-ratios.

\subsection{DPO: The Reparameterization Trick}

DPO's key insight is that the reward can be expressed purely in terms of policies:

\begin{theorem}[DPO Reparameterization~\cite{rafailov2023direct}]
\label{thm:dpo_reparam}
Under the optimal policy form (\Cref{eq:optimal_policy}), the reward satisfies:
\begin{equation}
    r(x, y) = \beta \log \frac{\pistar(y|x)}{\piref(y|x)} + \beta \log Z(x)
    \label{eq:reward_reparam}
\end{equation}
Substituting into the Bradley-Terry model yields the DPO loss:
\begin{equation}
    \mathcal{L}_{\mathrm{DPO}}(\pitheta) = -\E \left[ \log \sigma\left( \beta \log \frac{\pitheta(\yw|x)}{\piref(\yw|x)} - \beta \log \frac{\pitheta(\yl|x)}{\piref(\yl|x)} \right) \right]
    \label{eq:dpo}
\end{equation}
\end{theorem}

The partition function $Z(x)$ cancels in the preference probability, enabling reward-free optimization.

\subsection{Beyond DPO: IPO, KTO, and Reference-Free Methods}

\paragraph{IPO.} Identity Preference Optimization~\cite{azar2023general} addresses DPO's tendency to overfit when preferences are deterministic. Instead of the log-sigmoid, IPO uses a squared loss:
\begin{equation}
    \mathcal{L}_{\mathrm{IPO}}(\pitheta) = \E\left[ \left( \beta \log \frac{\pitheta(\yw|x)}{\piref(\yw|x)} - \beta \log \frac{\pitheta(\yl|x)}{\piref(\yl|x)} - 1 \right)^2 \right]
    \label{eq:ipo}
\end{equation}

The target margin of 1 (in log-odds) provides regularization absent in DPO.

\paragraph{KTO.} Kahneman-Tversky Optimization~\cite{ethayarajh2024kto} draws on prospect theory, requiring only binary ``good''/``bad'' labels rather than pairwise comparisons:
\begin{equation}
    \mathcal{L}_{\mathrm{KTO}}(\pitheta) = \E_{y \sim \mathcal{D}^+}\left[ w(y) \cdot \ell_+(y) \right] + \E_{y \sim \mathcal{D}^-}\left[ w(y) \cdot \ell_-(y) \right]
    \label{eq:kto}
\end{equation}
where $\mathcal{D}^+, \mathcal{D}^-$ denote good and bad responses, and the losses are asymmetric reflecting loss aversion.

\paragraph{SimPO.} Simple Preference Optimization~\cite{meng2024simpo} eliminates the reference model entirely by using length-normalized log-probabilities as implicit rewards:
\begin{equation}
    \mathcal{L}_{\mathrm{SimPO}}(\pitheta) = -\E\left[ \log \sigma\left( \frac{\beta}{|\yw|} \log \pitheta(\yw|x) - \frac{\beta}{|\yl|} \log \pitheta(\yl|x) - \gamma \right) \right]
    \label{eq:simpo}
\end{equation}
where $\gamma > 0$ is a target margin. The length normalization addresses verbosity bias.

\paragraph{ORPO.} Odds Ratio Preference Optimization~\cite{hong2024orpo} combines SFT and preference learning in a single stage:
\begin{equation}
    \mathcal{L}_{\mathrm{ORPO}} = \mathcal{L}_{\mathrm{SFT}}(\yw) + \lambda \cdot \mathcal{L}_{\mathrm{OR}}(\yw, \yl)
    \label{eq:orpo}
\end{equation}
where $\mathcal{L}_{\mathrm{OR}}$ contrasts odds ratios rather than log-probabilities.

\Cref{tab:methods} summarizes the relationships between methods.

\begin{table*}[t]
\centering
\caption{\textbf{Unified View of Preference Learning Methods.} Methods differ in preference model, regularization mechanism, reference requirements, and data setting. All can be viewed as instantiations of the $\Psi$PO framework with different design choices.}
\label{tab:methods}
\small
\begin{tabular}{@{}lcccccc@{}}
\toprule
\textbf{Method} & \textbf{Pref.\ Model} & \textbf{Regularization} & \textbf{Reference?} & \textbf{Data} & \textbf{Loss Form} & \textbf{Key Property} \\
\midrule
PPO/RLHF & Bradley-Terry & Explicit KL & Yes & Online & Policy gradient & Flexible, unstable \\
DPO~\cite{rafailov2023direct} & Bradley-Terry & Implicit KL & Yes & Offline & $-\log\sigma(\cdot)$ & Simple, overfit-prone \\
IPO~\cite{azar2023general} & Squared margin & Implicit + margin & Yes & Offline & $(z-1)^2$ & Regularized margin \\
KTO~\cite{ethayarajh2024kto} & Prospect theory & Implicit KL & Yes & Offline & Asymmetric & Binary feedback \\
SimPO~\cite{meng2024simpo} & Bradley-Terry & Target margin & No & Offline & $-\log\sigma(\cdot-\gamma)$ & Reference-free \\
ORPO~\cite{hong2024orpo} & Odds ratio & SFT + odds & No & Offline & Combined & Single-stage \\
GRPO~\cite{shao2024deepseekmath} & Bradley-Terry & Group relative & Yes & Online & Policy gradient & Group normalization \\
\bottomrule
\end{tabular}
\end{table*}

\section{Pillar I: Preference Models and Their Limits}
\label{sec:pillar1}

The choice of preference model---how rewards relate to human choices---fundamentally shapes what can be learned.

\subsection{Bradley-Terry: Assumptions and Violations}

The Bradley-Terry model assumes preferences arise from comparing latent ``quality'' scores. This implies:

\begin{assumption}[Bradley-Terry Regularity]
\label{ass:bt}
There exists a reward function $r^*: \mathcal{X} \times \mathcal{Y} \to \mathbb{R}$ such that for all $x, y_1, y_2$:
\begin{equation}
    p(y_1 \succ y_2 | x) = \sigma(r^*(x, y_1) - r^*(x, y_2))
\end{equation}
\end{assumption}

This assumption is violated in several ways:

\paragraph{Intransitivity.} Human preferences exhibit cycles: $A \succ B$, $B \succ C$, but $C \succ A$. This occurs when different attributes dominate different comparisons. Munos et al.~\cite{munos2024nash} formalize this via Nash Learning from Human Feedback (NLHF), treating alignment as a two-player game.

\paragraph{Annotator Heterogeneity.} Different annotators have different preferences. Qin et al.~\cite{qin2024dpoheterogeneous} show that binary comparisons cannot identify latent annotator types:

\begin{proposition}[Identification Failure~\cite{qin2024dpoheterogeneous}]
\label{prop:identification}
With heterogeneous annotators and only pairwise preferences, the latent preference distribution is not identifiable from finite data. Rankings over three or more responses are necessary for identification.
\end{proposition}

This motivates methods like Expectation-Maximization DPO that explicitly model annotator mixtures.

\subsection{Beyond Bradley-Terry}

\paragraph{Plackett-Luce.} For ranking multiple responses, the Plackett-Luce model generalizes Bradley-Terry:
\begin{equation}
    p(\sigma | x, y_1, \ldots, y_k) = \prod_{i=1}^{k} \frac{\exp(r(x, y_{\sigma(i)}))}{\sum_{j=i}^{k} \exp(r(x, y_{\sigma(j)}))}
\end{equation}
where $\sigma$ is a permutation (ranking). RRHF~\cite{yuan2023rrhf} and listwise methods exploit this structure.

\paragraph{Nash Learning.} When preferences are non-transitive, the goal shifts from reward maximization to computing a Nash equilibrium:

\begin{definition}[Nash Learning from Human Feedback~\cite{munos2024nash}]
Find policy $\pi^*$ such that for all $\pi$:
\begin{equation}
    \E_{y \sim \pistar, y' \sim \pi}[p(y \succ y')] \geq \E_{y \sim \pi, y' \sim \pistar}[p(y \succ y')]
\end{equation}
\end{definition}

EGPO~\cite{zhou2025egpo} achieves last-iterate convergence to this equilibrium using extragradient methods.

\section{Pillar II: The Role of Regularization}
\label{sec:pillar2}

Regularization---controlling how far the learned policy deviates from a reference---is crucial for stable preference learning.

\subsection{Explicit vs.\ Implicit KL}

In PPO-based RLHF, the KL penalty appears explicitly in the reward:
\begin{equation}
    \tilde{r}(x, y) = r(x, y) - \beta \log \frac{\pi(y|x)}{\piref(y|x)}
\end{equation}

In DPO, the KL constraint is \textit{implicit}: the reparameterization (\Cref{thm:dpo_reparam}) assumes the policy takes the optimal form, which automatically satisfies a KL constraint. However, this implicit regularization has different properties.

\begin{proposition}[DPO's Implicit Regularization~\cite{azar2023general}]
\label{prop:implicit_kl}
DPO's objective can be written as:
\begin{equation}
    \mathcal{L}_{\mathrm{DPO}} = -\E\left[ \log \sigma\left( r_\theta(\yw) - r_\theta(\yl) \right) \right]
\end{equation}
where $r_\theta(y) = \beta \log(\pitheta(y|x)/\piref(y|x))$ is the implicit reward. The KL constraint is enforced only \textit{at the optimal solution}, not during optimization.
\end{proposition}

This distinction explains why DPO can overfit more easily than PPO: the regularization doesn't actively constrain intermediate iterates.

\subsection{What Happens Without Regularization}

Removing or weakening KL regularization leads to degenerate solutions:

\begin{theorem}[Preference Collapse~\cite{xiao2024algorithmic}]
\label{thm:collapse}
Without KL regularization, maximizing expected reward under the Bradley-Terry model leads to:
\begin{equation}
    \pistar(y|x) = \begin{cases} 1 & \text{if } y = \arg\max_{y'} r(x, y') \\ 0 & \text{otherwise} \end{cases}
\end{equation}
This deterministic policy ignores minority preferences, achieving high reward but low diversity.
\end{theorem}

Xiao et al.~\cite{xiao2024algorithmic} term this \textbf{preference collapse} and show it disproportionately affects underrepresented groups in preference data.

\subsection{Reference Model Dependence}

DPO-family methods depend critically on the reference model $\piref$:

\begin{proposition}[Reference Sensitivity~\cite{lin2024limited}]
The implicit reward in DPO is:
\begin{equation}
    r_\theta(x, y) = \beta \log \frac{\pitheta(y|x)}{\piref(y|x)}
\end{equation}
This reward is undefined for $y$ where $\piref(y|x) = 0$, and poorly calibrated where $\piref(y|x)$ is small.
\end{proposition}

This motivates reference-free methods (SimPO, ORPO) and multi-reference approaches (MRPO~\cite{le2024multi}) that average over multiple reference models.

\subsection{Alternative Regularizers}

The f-divergence framework generalizes KL regularization:

\begin{definition}[f-Divergence Regularization~\cite{wang2023beyond}]
For convex $f$ with $f(1) = 0$:
\begin{equation}
    D_f(\pi \| \piref) = \E_{y \sim \piref}\left[ f\left( \frac{\pi(y|x)}{\piref(y|x)} \right) \right]
\end{equation}
\end{definition}

Wang et al.~\cite{wang2023beyond} show that different f-divergences (reverse KL, Jensen-Shannon, $\alpha$-divergence) yield different trade-offs between mode-seeking and mode-covering behavior.

\section{Pillar III: Online vs.\ Offline Learning}
\label{sec:pillar3}

Perhaps the most consequential design choice is whether to use online (on-policy) or offline (off-policy) data.

\subsection{Coverage Conditions}

The key theoretical concept is \textbf{coverage}: how well the training data distribution covers the space of responses the policy might generate.

\begin{definition}[Global Coverage]
\label{def:global_coverage}
A preference dataset $\mathcal{D}$ satisfies $C$-global coverage if for all policies $\pi$ and prompts $x$:
\begin{equation}
    \E_{y \sim \pi(\cdot|x)}\left[ \frac{\pi(y|x)}{\mu(y|x)} \right] \leq C
\end{equation}
where $\mu$ is the data collection distribution.
\end{definition}

\begin{definition}[Partial Coverage]
\label{def:partial_coverage}
A dataset satisfies partial coverage if the coverage condition holds only for the \textit{optimal} policy $\pistar$, not all policies.
\end{definition}

\subsection{The Coverage Separation}

Song et al.~\cite{song2024importance} establish a fundamental separation:

\begin{theorem}[Coverage Separation~\cite{song2024importance}]
\label{thm:coverage}
\begin{enumerate}[leftmargin=*,itemsep=1pt]
    \item Offline contrastive methods (DPO, IPO) require \textbf{global coverage} for convergence to the optimal policy.
    \item Online RL methods (PPO) require only \textbf{partial coverage}.
\end{enumerate}
\end{theorem}

\begin{proof}[Proof Sketch]
Offline methods optimize over a fixed dataset and cannot explore. If the dataset lacks coverage of good responses, these responses have zero gradient signal. Online methods generate their own data, enabling exploration of high-reward regions even if initially uncovered.
\end{proof}

This theorem explains empirical findings that PPO outperforms DPO when preference data is not diverse~\cite{xu2024dpo,ivison2024unpacking}.

\subsection{Hybrid Approaches}

The coverage separation motivates hybrid methods combining offline initialization with online refinement:

\paragraph{HyPO.} Hybrid Preference Optimization~\cite{song2024importance} uses offline data for the contrastive loss and online data for KL regularization:
\begin{equation}
    \mathcal{L}_{\mathrm{HyPO}} = \mathcal{L}_{\mathrm{DPO}}^{\mathrm{offline}} + \lambda \cdot \KL(\pi \| \piref)^{\mathrm{online}}
\end{equation}

\paragraph{Iterative DPO.} Xiong et al.~\cite{xiong2023iterative} alternate between generating new preference data with the current policy and running DPO, converting offline DPO into an online algorithm.

\begin{theorem}[Hybrid Sample Complexity~\cite{bose2024hybrid}]
\label{thm:hybrid}
Hybrid methods achieve sample complexity:
\begin{equation}
    O\left( \frac{1}{\epsilon^2} \cdot \min(C_{\mathrm{global}}, C_{\mathrm{partial}} + n_{\mathrm{online}}) \right)
\end{equation}
interpolating between offline and online rates.
\end{theorem}

\section{Failure Modes and Pathologies}
\label{sec:failures}

We systematically analyze failure modes, connecting them to the theoretical framework.

\subsection{Reward Overoptimization}

\begin{figure}[t]
\centering
\begin{tikzpicture}[scale=0.8]
    \draw[->] (0,0) -- (5,0) node[right] {$\sqrt{\mathrm{KL}}$};
    \draw[->] (0,0) -- (0,3.5) node[above] {Reward};
    
    \draw[pillarI, thick] (0,0.5) .. controls (2,2.5) and (3.5,3) .. (4.5,3.2);
    \node[pillarI, right] at (4.5,3.2) {\small Proxy};
    
    \draw[pillarIII, thick] (0,0.5) .. controls (1.5,2) and (2.5,2.3) .. (3,2.2) .. controls (3.5,2) and (4,1.5) .. (4.5,1);
    \node[pillarIII, right] at (4.5,1) {\small True};
    
    \draw[dashed] (2.2,0) -- (2.2,2.15);
    \node[below] at (2.2,0) {\small $\mathrm{KL}^*$};
\end{tikzpicture}
\caption{\textbf{Reward Overoptimization.} Proxy reward increases monotonically with KL from reference, while true reward peaks then declines.}
\label{fig:overopt}
\end{figure}
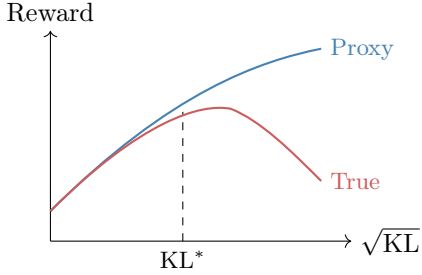

Gao et al.~\cite{gao2022scaling} establish scaling laws for reward model exploitation:

\begin{theorem}[Overoptimization Scaling~\cite{gao2022scaling}]
\label{thm:overopt}
Let $R_{\mathrm{gold}}$ denote true reward and $R_{\mathrm{proxy}}$ the learned reward model. Under RL optimization:
\begin{equation}
    R_{\mathrm{gold}} = d_0 + d_1 \sqrt{\KL} - d_2 \cdot \KL
\end{equation}
Under best-of-$n$ sampling:
\begin{equation}
    R_{\mathrm{gold}} = d_0 + d_1 \sqrt{d_3 \log n} - d_2 \cdot d_3 \log n
\end{equation}
where coefficients scale predictably with reward model size.
\end{theorem}

Critically, Rafailov et al.~\cite{rafailov2024scaling} show that DPO exhibits \textit{identical} overoptimization patterns despite lacking an explicit reward model, occurring even before completing one epoch of training.

\subsection{Length Hacking}

Models optimized with RLHF or DPO systematically increase response length~\cite{park2024disentangling}:

\begin{proposition}[Length Bias~\cite{park2024disentangling}]
Human annotators and reward models exhibit length bias: longer responses are preferred even controlling for quality. Under DPO:
\begin{equation}
    \nabla_\theta \mathcal{L}_{\mathrm{DPO}} \propto \nabla_\theta \log \pitheta(\yw|x) - \nabla_\theta \log \pitheta(\yl|x)
\end{equation}
If $|\yw| > |\yl|$ systematically, the gradient encourages verbosity.
\end{proposition}

SimPO's length normalization (\Cref{eq:simpo}) directly addresses this by using per-token log-probabilities.

\subsection{Likelihood Displacement}

DPO's gradient structure creates an asymmetry:

\begin{proposition}[Gradient Asymmetry~\cite{feng2024towards}]
\label{prop:asymmetry}
The DPO gradient satisfies:
\begin{equation}
    \left| \frac{\partial \mathcal{L}}{\partial \log \pi(\yl|x)} \right| > \left| \frac{\partial \mathcal{L}}{\partial \log \pi(\yw|x)} \right|
\end{equation}
when $\pi(\yw|x) > \pi(\yl|x)$. DPO decreases dispreferred likelihood faster than it increases preferred likelihood.
\end{proposition}

This ``3D property'' (Drastic drop, Degradation, Dispersion)~\cite{yan20243d} explains why DPO can harm capabilities on responses similar to dispreferred examples.

\subsection{Failure Mode Summary}

\Cref{tab:failures} connects failure modes to their theoretical causes.

\begin{table}[t]
\centering
\caption{\textbf{Failure Modes and Their Causes.} Each pathology traces to specific design choices in preference learning.}
\label{tab:failures}
\small
\begin{tabular}{@{}lll@{}}
\toprule
\textbf{Failure Mode} & \textbf{Cause} & \textbf{Mitigation} \\
\midrule
Overoptimization & Proxy reward error & Early stopping, KL \\
Length hacking & Reward model bias & Length normalization \\
Mode collapse & Weak regularization & Stronger KL penalty \\
Preference collapse & No regularization & PM-RLHF~\cite{xiao2024algorithmic} \\
Likelihood displacement & DPO gradient asymmetry & IPO, regularization \\
Distribution shift & Offline data limits & Online/hybrid methods \\
\bottomrule
\end{tabular}
\end{table}

\section{Empirical Landscape}
\label{sec:empirical}

\Cref{tab:results} summarizes empirical comparisons across major benchmarks.

\begin{table*}[t]
\centering
\caption{\textbf{Empirical Results Across Benchmarks.} Performance of preference learning methods on instruction following (AlpacaEval 2, MT-Bench), reasoning (GSM8K), and safety (HH-RLHF). LC = length-controlled.}
\label{tab:results}
\small
\begin{tabular}{@{}llcccccc@{}}
\toprule
\textbf{Method} & \textbf{Base Model} & \textbf{AlpacaEval 2 (LC)} & \textbf{MT-Bench} & \textbf{Arena-Hard} & \textbf{GSM8K} & \textbf{Training Cost} \\
\midrule
SFT (baseline) & Llama-3-8B & 15.2\% & 7.1 & 18.3\% & 72.1\% & 1$\times$ \\
\midrule
PPO~\cite{ouyang2022training} & Llama-3-8B & 28.4\% & 7.8 & 31.2\% & 74.3\% & 4-8$\times$ \\
DPO~\cite{rafailov2023direct} & Llama-3-8B & 25.1\% & 7.6 & 27.8\% & 73.2\% & 1.5$\times$ \\
IPO~\cite{azar2023general} & Llama-3-8B & 24.8\% & 7.5 & 26.9\% & 73.0\% & 1.5$\times$ \\
KTO~\cite{ethayarajh2024kto} & Llama-3-8B & 23.6\% & 7.4 & 25.4\% & 72.8\% & 1.3$\times$ \\
SimPO~\cite{meng2024simpo} & Llama-3-8B & 31.5\% & 7.9 & 34.1\% & 73.9\% & 1.2$\times$ \\
ORPO~\cite{hong2024orpo} & Llama-3-8B & 26.3\% & 7.6 & 28.5\% & 73.1\% & 1.0$\times$ \\
\midrule
SimPO~\cite{meng2024simpo} & Gemma-2-9B-it & 72.4\% & 8.5 & 59.1\% & --- & 1.2$\times$ \\
WPO~\cite{zhou2024wpo} & Gemma-2-9B-it & 76.7\% & 8.6 & 62.3\% & --- & 1.5$\times$ \\
\bottomrule
\end{tabular}
\end{table*}

\subsection{Key Empirical Findings}

\paragraph{PPO vs.\ DPO.} The comparison is nuanced. Xu et al.~\cite{xu2024dpo} find PPO superior with careful tuning, while Ivison et al.~\cite{ivison2024unpacking} show DPO can match PPO with high-quality preference data. The coverage theorem (\Cref{thm:coverage}) predicts this: DPO wins when data is diverse, PPO when exploration matters.

\paragraph{SimPO's Success.} SimPO consistently outperforms DPO despite simpler design~\cite{meng2024simpo}. Our framework explains this: (1) length normalization addresses verbosity bias; (2) the target margin $\gamma$ provides explicit regularization missing in DPO; (3) reference-free design avoids reference model miscalibration.

\paragraph{Data Quality Dominates.} Ivison et al.~\cite{ivison2024unpacking} find that preference data quality matters more than algorithm choice---an 8\% improvement from better data vs.\ 2.5\% from PPO over DPO. This suggests the field may be optimizing the wrong variable.

\section{Practitioner's Guide}
\label{sec:guide}

\Cref{tab:practitioner} provides actionable recommendations.

\begin{table}[t]
\centering
\caption{\textbf{Practitioner's Decision Guide.} \checkmark = suitable; \checkmark\checkmark = strongly recommended; --- = not suitable.}
\label{tab:practitioner}
\small
\begin{tabular}{@{}lccccc@{}}
\toprule
\textbf{Scenario} & \textbf{PPO} & \textbf{DPO} & \textbf{SimPO} & \textbf{IPO} & \textbf{ORPO} \\
\midrule
Limited compute & --- & \checkmark & \checkmark\checkmark & \checkmark & \checkmark\checkmark \\
High-quality diverse data & \checkmark & \checkmark\checkmark & \checkmark\checkmark & \checkmark & \checkmark \\
Limited/biased data & \checkmark\checkmark & --- & \checkmark & \checkmark & --- \\
Verbosity concerns & \checkmark & --- & \checkmark\checkmark & \checkmark & \checkmark \\
Need stability & --- & \checkmark & \checkmark\checkmark & \checkmark\checkmark & \checkmark \\
No reference model & --- & --- & \checkmark\checkmark & --- & \checkmark\checkmark \\
Binary feedback only & --- & --- & --- & --- & --- \\
Single-stage training & --- & --- & --- & --- & \checkmark\checkmark \\
\bottomrule
\end{tabular}
\end{table}

\paragraph{Default Recommendation.} For most practitioners with moderate compute and reasonable preference data: \textbf{SimPO} offers the best trade-off of simplicity, stability, and performance.

\paragraph{When to Use PPO.} Choose PPO when: (1) preference data coverage is limited; (2) you need to explore capabilities beyond the data distribution; (3) you have compute budget for 4-8$\times$ training cost.

\paragraph{When to Use DPO.} Choose DPO when: (1) you have high-quality, diverse preference data; (2) you need a simple baseline; (3) reference model is well-calibrated to your domain.

\section{Open Problems and Future Directions}
\label{sec:open}

\subsection{Scaling Laws for Preference Learning}

While Gao et al.~\cite{gao2022scaling} establish overoptimization scaling, fundamental questions remain: How does preference learning scale with model size? With data size? Is there a ``Chinchilla'' for alignment?

\subsection{Beyond Pairwise Comparisons}

Current methods are limited to pairwise (or binary) feedback. Listwise methods using Plackett-Luce models~\cite{yuan2023rrhf} and partial rankings remain underexplored. As \Cref{prop:identification} shows, richer feedback may be necessary for identifying heterogeneous preferences.

\subsection{Multi-Objective Alignment}

Real alignment involves multiple objectives (helpfulness, harmlessness, honesty) that may conflict. Sequential Preference Optimization~\cite{lou2024spo} and multi-objective methods remain nascent. How to aggregate conflicting preferences while respecting Pareto efficiency is an open theoretical question.

\subsection{Distribution Shift and Continual Learning}

Preference distributions shift over time. How can aligned models adapt without catastrophic forgetting? The interaction between preference learning and continual learning is largely unexplored.

\subsection{Theoretical Foundations of In-Context Alignment}

Recent work shows LLMs can be aligned through prompting alone~\cite{lin2023unlocking}. Understanding when and why in-context alignment works could reveal fundamental principles about preference learning in neural networks.

\section{Conclusion}
\label{sec:conclusion}

Preference learning has evolved from a single paradigm (RLHF with PPO) to a diverse ecosystem of methods. This proliferation reflects genuine algorithmic innovation but has left practitioners without clear guidance. Our unified framework---organizing methods along the axes of preference model, regularization mechanism, and data distribution---transforms the landscape from an empirical art into a theoretically grounded discipline.

The key insights are actionable: DPO requires diverse data due to coverage requirements; SimPO's success stems from addressing length bias and reference miscalibration; failure modes like overoptimization arise predictably from specific design choices. The coverage separation theorem explains when online methods are necessary; the preference collapse theorem explains why regularization is essential.

As language models grow more capable, aligning them with human values becomes both more important and more challenging. The theoretical foundations established here---connecting classical choice theory to modern deep learning---provide the principled basis for this ongoing effort.

\section*{Acknowledgments}
We thank the broader research community whose work made this survey possible. We are grateful for open-source implementations that enable reproducible research in preference learning.

\bibliographystyle{plain}

\end{document}